\documentclass{article}

\usepackage{PRIMEarxiv}

\usepackage[utf8]{inputenc} % allow utf-8 input
\usepackage[T1]{fontenc}    % use 8-bit T1 fonts
\usepackage{hyperref}       % hyperlinks
\usepackage{url}            % simple URL typesetting
\usepackage{booktabs}       % professional-quality tables
\usepackage{amsfonts}       % blackboard math symbols
\usepackage{nicefrac}       % compact symbols for 1/2, etc.
\usepackage{microtype}      % microtypography
\usepackage{lipsum}
\usepackage{fancyhdr}       % header
\usepackage{graphicx}       % graphics
\graphicspath{{media/}}     % organize your images and other figures under media/ folder

\usepackage[round]{natbib}

\usepackage{amssymb} %maths
\usepackage{amsmath} %maths
\usepackage{amsthm}
\usepackage{graphicx} % DO NOT CHANGE THIS
\usepackage{color}
\usepackage{tikz}
\usepackage[]{algorithm}
\usepackage[]{algorithmic}
\newtheorem{theorem}{Theorem}

\newtheorem{lemma}{Lemma}
\newtheorem{defn}{Definition}

\DeclareMathOperator*{\argmin}{arg\,min}
\DeclareMathOperator*{\argmax}{arg\,max}
\def\O{\mathcal{O}}
\def\K{\mathcal{K}}
\def\R{\mathbb{R}}

\title{Fast First-Order Methods for Monotone Strongly DR-Submodular Maximization}
\author{
  Omid Sadeghi \\
  University of Washington\\
  \texttt{omids@uw.edu} \\
   \And
  Maryam Fazel \\
  University of Washington \\
  \texttt{mfazel@uw.edu} \\
}

\begin{document}
\maketitle

\begin{abstract}
  Continuous DR-submodular functions are a class of functions that satisfy the Diminishing Returns (DR) property, which implies that they are concave along non-negative directions. Existing works have studied monotone continuous DR-submodular maximization subject to a convex constraint and have proposed efficient algorithms with approximation guarantees. However, in many applications, e.g., computing the stability number of a graph and mean-field inference for probabilistic log-submodular models, the DR-submodular function has the additional property of being \emph{strongly} concave along non-negative directions that could be utilized for obtaining faster convergence rates. In this paper, we first introduce and characterize the class of \emph{strongly DR-submodular} functions and show how such a property implies strong concavity along non-negative directions. Then, we study $L$-smooth monotone strongly DR-submodular functions that have bounded curvature, and we show how to exploit such additional structure to obtain algorithms with improved approximation guarantees and faster convergence rates for the maximization problem. In particular, we propose the SDRFW algorithm that matches the provably optimal $1-\frac{c}{e}$ approximation ratio after only $\lceil\frac{L}{\mu}\rceil$ iterations, where $c\in[0,1]$ and $\mu\geq 0$ are the curvature and the strong DR-submodularity parameter. Furthermore, we study the Projected Gradient Ascent (PGA) method for this problem and provide a refined analysis of the algorithm with an improved $\frac{1}{1+c}$ approximation ratio (compared to $\frac{1}{2}$ in prior works) and a linear convergence rate. Given that both algorithms require knowledge of the smoothness parameter $L$, we provide a \emph{novel} characterization of $L$ for DR-submodular functions showing that in many cases, computing $L$ could be formulated as a convex optimization problem, i.e., a geometric program, that could be solved efficiently. Experimental results illustrate and validate the efficiency and effectiveness of our algorithms.
\end{abstract}

% keywords can be removed
%\keywords{First keyword \and Second keyword \and More}
\section{Introduction}\label{sec:1}
Submodular set functions are a class of discrete functions that exhibit the Diminishing Returns (DR) property. A set function $F$ defined over the ground set $V$ is called submodular if for all items $j\in V$ and $A\subseteq B\subseteq V\setminus \{j\}$, we have:
\begin{equation*}
    F(A\cup \{j\})-F(A)\geq F(B\cup \{j\})-F(B).
\end{equation*}
In other words, the gain of adding a particular element $j$ to a set decreases as the set gets larger. Such functions are commonly used to quantify coverage and diversity in discrete domains. Submodular set function optimization has found many applications in machine learning such as viral marketing \citep{kempe2003maximizing}, dictionary learning \citep{das2011submodular}, feature selection for classification \citep{krause2007near} and document summarization \citep{lin2011class} to name a few.\\
While the DR property is mostly associated with submodular set functions, it can be defined similarly for continuous functions. A differentiable function $f:\mathcal{K} \to \mathbb{R}$, $\mathcal{K} \subseteq \mathbb{R}_+^n$, satisfies the DR property if for all $x,y\in \K$ such that $x_i\leq y_i~\forall i\in[n]$, we have $\nabla_if(x)\geq \nabla_i f(y)$ for all $i\in[n]$, i.e., $\nabla f$ is an order-reversing mapping. Continuous functions that satisfy the DR property are called DR-submodular. Similar to submodular set functions, continuous DR-submodular functions find applications in multiple domains such as influence and revenue maximization, MAP inference for DPP (Determinantal Point Process) and mean-field inference of probabilistic graphical models (see \citet[Section 6]{bian2020continuous} for more applications and details). While DR-submodular functions are generally non-convex/non-concave, the DR property provides a natural structure that allows designing tractable approximation algorithms. In particular, DR-submodular functions are concave along non-negative directions \citep{bian2017guaranteed}, i.e., for all $x,y$ such that $x_i\leq y_i~\forall i\in[n]$, we have $f(y)\leq f(x)+\langle \nabla f(x),y-x\rangle$.\\
Monotone DR-submodular maximization subject to a convex constraint has been previously studied in the literature. \citet{bian2017guaranteed} proposed a Frank-Wolfe variant that obtains a provably optimal $1-\frac{1}{e}$ approximation guarantee at a sub-linear convergence rate. Later, \citet{hassani2017gradient} studied the well-known Projected Gradient Ascent (PGA) method for this problem and proved that PGA has a $\frac{1}{2}$ approximation ratio and sub-linear rate of convergence.\\
A number of the DR-submodular objective functions in the aforementioned applications of the continuous DR property are indeed strongly concave along non-negative directions. For instance, consider a graph $G=(V,E)$ with adjacency matrix $A$. Computing the stability number $s(G)$ of the graph (i.e., cardinality of the largest subset of vertices such that no two vertices in this subset are adjacent) is a well-known NP-hard combinatorial problem. This problem could be formulated as $s(G)^{-1}=\min_{x\in \Delta}x^T (A+I) x$ where $\Delta=\{x\in \R^{|V|}:1^Tx=1,x_v\geq 0~\forall v\in V\}$ is the standard simplex and $I$ is the identity matrix \citep{motzkin1965maxima}. We can rewrite the problem as:
\begin{equation*}
    s(G)^{-1}=\min_{x\in \Delta}x^T (A+I) x=-\max_{x\in \Delta}x^T (-A-I) x=2-\max_{x\in \Delta}x^T (-A-I) x+\underbrace{21^Tx}_{=2}.
\end{equation*}
It will soon be clear that the function $x^T (-A-I) x+21^Tx$ is $2$-strongly DR-submodular (see Definition \ref{def:1} in Section \ref{sec:2}) and thus, finding the stability number of a graph could be formulated as a convex-constrained monotone $2$-strongly DR-submodular maximization problem. Another example is the mean-field inference problem for probabilistic log-submodular models (Section 6.6 of \citet{bian2020continuous}). Let $F:2^V\to \R$ be a submodular set function. Consider distributions over subsets $S\subseteq V$ of the form $P(S)=\frac{1}{Z}\text{exp}(F(S))$ where the normalizing quantity $Z=\sum_{S\subseteq V}\text{exp}(F(S))$ is called the \emph{partition function}. In general, computing the partition function involves summing over exponential number of terms and is not computationally feasible. Alternatively, one can use mean-field inference to approximate $P(S)$ by a completely factorized distribution $Q(S)$, i.e., elements $i\in V$ are picked independently and $Q(S|x)=\prod_{i\in S}x_i \prod_{j\notin S}(1-x_j)$ where $x\in[0,1]^{V}$ is the vector of marginals, by minimizing the KL divergence $\textbf{KL}(x)=\sum_{S\subseteq V}Q(S|x)\ln \big(\frac{Q(S|x)}{p(S)}\big)$ between the two distributions. $\textbf{KL}(x)$ can be written as:
\begin{equation*}
    \textbf{KL}(x)=-\sum_{S\subseteq V}F(S)\prod_{i\in S}x_i \prod_{j\notin S}(1-x_j)+\sum_{i=1}^{|V|}\big(x_i\ln x_i+(1-x_i)\ln (1-x_i)\big)+\ln Z.
\end{equation*}
It will be clear later that the problem $\max_{x\in[0,1]^{V}}-\textbf{KL}(x)$ (equivalent to minimizing $\textbf{KL}(x)$) is a 4-strongly DR-submodular maximization problem.
\subsection{Contributions}
In this paper, we precisely characterize the class of \emph{strongly DR-submodular} functions, i.e., DR-submodular functions that are strongly concave along non-negative directions. We consider the optimization problem $\max_{x\in \K}f(x)$ where $\K$ is a convex set and $f$ is a monotone, smooth and strongly DR-submodular function with bounded curvature $c_f\in[0,1]$, and we provide algorithms with refined approximation guarantees that exploit the strong DR-submodularity structure of the objective function. Specifically, we make the following contributions:\\
% $\bullet$ In Section \ref{sec:2}, we provide equivalent definitions of strongly DR-submodular functions (Definition \ref{def:1}) and we show how such a property implies strong concavity along non-negative directions (Lemma \ref{lemma1}).\\
% $\bullet$ We introduce a notion of \emph{curvature} $c_f\in [0,1]$ for a general monotone continuous function $f$ in Section \ref{sec:2} (Definition \ref{def:2}) as a measure of how far the function is from being linear. We show how our definition extends the notion of total curvature of submodular set functions \citep{conforti1984submodular} to general continuous functions and we use it to refine our approximation guarantees.\\
$\bullet$ We propose the SDRFW algorithm in Section \ref{sec:3} that obtains the approximation guarantee $1-\frac{c_f}{e}$ after $\lceil\frac{L}{\mu}\rceil$ iterations, where $L$ is the smoothness parameter and $\mu$ is the strong DR-submodularity parameter of the function $f$.\\
$\bullet$ In Section \ref{sec:4}, we analyze PGA for our problem and we provide a refined and sharper analysis of the algorithm showing that PGA has an improved $\frac{1}{1+c_f}$ approximation guarantee (compared to $\frac{1}{2}$ in prior works) at a linear convergence rate, i.e., it only takes $\O(\ln(\frac{1}{\epsilon}))$ iterations to get within $\epsilon$ of $\frac{1}{1+c_f}\text{OPT}$, where $\text{OPT}$ denotes the optimal value of the problem.\\
$\bullet$ We also study \emph{online} strongly DR-submodular maximization in Section \ref{sec:4-1}. We analyze the online counterpart of PGA, called Online Gradient Ascent (OGA), for this problem and we show how our techniques could be used to obtain improved logarithmic bounds for the $(\frac{1}{1+c})$-regret of the algorithm.\\
$\bullet$ Given that both SDRFW and PGA require knowledge of the smoothness parameter $L$, in Section \ref{sec:5}, we provide a \emph{novel} characterization of $L$ for (strongly) DR-submodular functions showing that in many cases, computing $L$ could be formulated as a convex optimization problem, i.e., a geometric program, that could be solved efficiently. Therefore, we can compute $L$ in an efficient manner before running the algorithms.\\
Finally, we test our algorithms on a class of strongly DR-submodular functions and also for the problem of computing the stability number of graphs in Section \ref{sec:6}. All missing proofs are provided in the Appendix.
\section{Preliminaries}\label{sec:2}
\textbf{Notation.} The set $\{1,2,\dots,n\}$ is denoted by $[n]$. For a vector $x\in \R^n$, $x_i$ is used to denote the $i$-th entry of $x$. Similarly, for a matrix $A\in \R^{n\times n}$, we use $A_{ij}$ to indicate the entry in the $i$-th row and $j$-th column of the matrix. The inner product of two vectors $x,y\in\mathbb{R}^n$ is denoted by either $\langle x, y \rangle$ or $x^T y$. Moreover, for two vectors $x,y\in \mathbb{R}^n$, we have $x\preceq y$ if $x_i \leq y_i$ holds for every $i\in[n]$. Also, we use $x\vee y=\max\{x,y\}$ and $x\wedge y=\min\{x,y\}$ to denote the component-wise maximum and minimum of $x,y\in \R^n$, i.e., for all $i\in[n]$, $[x\vee y]_i=\max\{x_i,y_i\}$ and $[x\wedge y]_i=\min\{x_i,y_i\}$ holds. $\|\cdot\|$ indicates the Euclidean norm by default. We use $\text{Proj}_{\K}(x)$ to denote the Euclidean projection of $x$ onto the convex set $\K$, i.e., $\text{Proj}_{\K}(x)=\argmin_{z\in \K}\|z-x\|$. A function $f:\K \to \R$ is called $\beta$-Lipschitz if for all $x,y\in \K$, we have $|f(y)-f(x)|\leq \beta \|y-x\|$. Also, $f$ is called monotone if for all $x,y\in \K$ such that $x\preceq y$, $f(x)\leq f(y)$ holds. The diameter of a set $\K$ is defined as $R=\max_{x,y\in \K}\|y-x\|$. By default, $f(\cdot)$ indicates a continuous function and $F(\cdot)$ denotes a discrete set function.\\
\textbf{DR-submodular functions.} A differentiable function $f:\K \rightarrow \mathbb{R}$, $\K\subset \mathbb{R}_+^n$, is called DR-submodular if for all $x,y$ such that $x\preceq y$, $\nabla f(x) \succeq \nabla f(y)$ holds. If $f$ is twice differentiable, DR-submodularity is equivalent to the Hessian matrix $\nabla^2 f(x)$ being entry-wise non-positive for all $x\in \K$. For instance, for a submodular set function $F$ over the ground set $V$, the multilinear extension $f:[0,1]^{|V|} \rightarrow \mathbb{R}$ of $F$ defined as 
$f(x)=\sum_{S\subset V}F(S)\prod_{i\in S}x_i \prod_{j\notin S}(1-x_j)$ is DR-submodular. While DR-submodularity and concavity are equivalent for the special case of $n=1$, DR-submodular functions are generally non-concave. Nonetheless, an important consequence of DR-submodularity is concavity along non-negative directions \citep{bian2017guaranteed}, i.e., for all $x,y$ such that $x\preceq y$, we have $f(y)\leq f(x)+\langle \nabla f(x),y-x\rangle$.\\
\textbf{Strongly DR-submodular functions.} We define the class of strongly DR-submodular functions below.
\begin{defn}\label{def:1}
For $\mu\geq 0$, we call a differentiable function $f:\K \rightarrow \mathbb{R}$, $\K\subset \mathbb{R}_+^n$, to be $\mu$-strongly DR-submodular if any of the following equivalent properties hold:\\
$\bullet$ $f(\cdot)+\frac{\mu}{2}\|\cdot\|^2$ is DR-submodular.\\
$\bullet$ For all $x,y\in \K$ such that $x\preceq y$, we have
\begin{equation*}
\nabla f(x)\succeq \nabla f(y)+\mu(y-x).
\end{equation*}
$\bullet$ If $f$ is twice differentiable, $\nabla^2_{ii}f(x)\leq -\mu~\forall i\in[n]$ and $\nabla^2_{ij}f(x)\leq 0~\forall i\neq j$ holds for all $x\in \K$.
\end{defn}
As an example, consider the problem of computing the stability number of a graph (introduced in Section \ref{sec:1}). The Hessian of the objective function $x^T (-A-I) x+21^Tx$ is $H=-2A-2I$. Given that $H_{ii}=-2$ and $H_{ij}\leq 0$ for all $i\neq j\in [n]$, the objective function is $2$-strongly DR-submodular.\\
We can prove that $\mu$-strongly DR-submodular functions are indeed $\mu$-strongly concave along non-negative directions. This property is extensively used in the design and analysis of our algorithms.
\begin{lemma}\citep{sadeghi2021improved}\label{lemma1}
If $f:\K \rightarrow \mathbb{R}$, $\K\subset \mathbb{R}_+^n$, is differentiable and $\mu$-strongly DR-submodular, for all $x\in \mathcal{K}$ and $v\succeq 0$ or $v\preceq 0$, the following holds:
\begin{equation*}
    f(x+v)\leq f(x)+\langle \nabla f(x),v\rangle -\frac{\mu}{2}\|v\|^2.
\end{equation*}
\end{lemma}
\citet{bian2020continuous} defined $\mu$-strongly DR-submodular functions to be the class of functions that are $\mu$-strongly concave along non-negative directions. Note that our definition of $\mu$-strong DR-submodularity is a stronger condition than $\mu$-strong concavity along non-negative directions. For instance, $\mu$-strongly concave functions are $\mu$-strongly concave along any direction, but they may not even be DR-submodular.\\
\textbf{Smooth functions.} A differentiable function $f:\K \rightarrow \mathbb{R}$, $\K\subset \mathbb{R}_+^n$, is called $L$-smooth if for all $x,y\in \K$, we have
\begin{equation*}
    f(y)\geq f(x)+\langle \nabla f(x),y-x\rangle-\frac{L}{2}\|y-x\|^2.
\end{equation*}
If $f$ is twice differentiable, there is an equivalent definition of smoothness: $f$ is $L$-smooth if $\nabla^2 f(x)\succeq -LI$ holds for all $x\in \K$ where $I$ is the identity matrix. In other words, the smallest eigenvalue of the Hessian of $f$ is uniformly lower bounded by $-L$ everywhere. Combining the definition of smooth functions and the result of Lemma \ref{lemma1}, it is clear that for a $\mu$-strongly DR-submodular and $L$-smooth function, $\mu\leq L$ holds.\\ 
\textbf{Curvature.} We define the notion of curvature for monotone continuous functions below.
\begin{defn}\label{def:2}
Given a monotone differentiable function $f:\K \rightarrow \mathbb{R}$, $\K\subset \mathbb{R}_+^n$, we define the curvature of $f$ as follows:
\begin{equation*}
    c_f=1-\inf_{x,y\in \K,i\in[n]}\frac{\nabla_if(y)}{\nabla_if(x)}.
\end{equation*}
If $f$ is DR-submodular and $0\in \K$, we have $c_f=1-\inf_{x\in \K,i\in[n]}\frac{\nabla_i f(x)}{\nabla_i f(0)}$.
\end{defn}
It is easy to see that $c_f\in[0,1]$ holds for all monotone $f$. $c_f\leq 1$ is due to monotonocity of $f$ (i.e., $\nabla f$ being non-negative) and $c_f\geq 0$ follows from setting $x=y$ in the definition. If $c_f=0$, $f$ is linear and larger $c_f$ corresponds to $f$ being more curved.\\
A similar notion of curvature was introduced in \citet{sessa2019bounding}. The definition is inspired by the curvature of submodular set functions \citep{conforti1984submodular}. In fact, if $f$ is the multilinear extension of a monotone submodular set function $F$, $c_f$ coincides with the curvature of $F$ \citep{sadeghi2021differentially}. Submodular set function maximization with bounded curvature has been widely studied in the literature. \citet{conforti1984submodular} showed that the greedy algorithm applied to the monotone submodular set function maximization problem subject to a cardinality constraint has a $\frac{1}{\kappa}(1-e^{-\kappa})$ approximation ratio, where $\kappa$ is the curvature of the set function. More recently, \citet{sviridenko2017optimal} proposed two approximation algorithms for the more general problem of monotone submodular maximization subject to a matroid constraint and obtained a $1-\frac{\kappa}{e}$ approximation ratio for these two algorithms. They also provided matching upper bounds for this problem showing that the $1-\frac{\kappa}{e}$ approximation ratio is indeed optimal. Later on, \citet{feldman2021guess} managed to obtain the same $1-\frac{\kappa}{e}$ approximation ratio with an algorithm that is much faster than the ones proposed by \citet{sviridenko2017optimal}.
\section{Strongly DR-submodular Frank-Wolfe (SDRFW) algorithm}\label{sec:3}
In this section, we propose the SDRFW algorithm for strongly DR-submodular maximization with bounded curvature. Throughout the section, we make the further assumption that the domain set $\K$ contains the origin, i.e., $0\in \K$. Furthermore, we consider $f$ to be a monotone, $\mu$-strongly DR-submodular and $L$-smooth function with curvature $c_f\in[0,1]$. Without loss of generality, we also assume that $f$ is normalized, i.e., $f(0)=0$. For the DR-submodular setting ($\mu=0$), \citet{bian2017guaranteed} proposed a Frank-Wolfe variant for solving the problem. Starting from $x_0=0$, their algorithm performs $K$ Frank-Wolfe updates where at each iteration $k\in\{0,\dots,K-1\}$, it finds $v_k$ such that $v_k=\argmax_{x\in \K}\langle x,\nabla f(x_k)\rangle$, performs the update $x_{k+1}=x_{k}+\frac{1}{K}v_k$ and outputs $x_{K}$.\\
Define $g(x)=f(x)-\ell^Tx$ where for all $i\in[n]$, $\ell_i=\min_x \nabla_i f(x)$. Note that similar to $f$, $g$ is also a normalized, monotone, $\mu$-strongly DR-submodular and $L$-smooth function.\\
% For instance, in order to verify the $\mu$-strong DR-submodularity of $g$, for all $x\in \K$ and $v\in \pm \R^n$, we can write:
% \begin{align*}
%     g(x+v)&=f(x+v)-\langle \ell, x+v\rangle \\
%     &\leq f(x)+\langle \nabla f(x),v\rangle -\frac{\mu}{2}\|v\|^2-\langle \ell, x+v\rangle\\
%     &= \underbrace{f(x)-\langle \ell, x\rangle}_{=g(x)} +\langle \underbrace{\nabla f(x)-\ell}_{=\nabla g(x)},v\rangle -\frac{\mu}{2}\|v\|^2\\
%     &= g(x)+\langle \nabla g(x),v\rangle -\frac{\mu}{2}\|v\|^2,
% \end{align*}
% where the inequality uses the $\mu$-strong DR-submodularity of $f$.\\
The SDRFW algorithm is presented in Algorithm \ref{alg1}. First, note that the output of the algorithm ($x=x_{K}$) is the average of $K$ points $\{v_k\}_{k=0}^{K-1}$ in the convex domain set $\K$, and therefore, $x\in \K$. Also, it is noteworthy that the update rule for $\{v_k\}_{k=1}^K$ can be computed efficiently in many cases. To see this, for all $k\in\{0,\dots,K-1\}$, we can equivalently write:
\begin{equation*}
    v_k=\text{Proj}_{\K}\big(\frac{1}{\mu}\nabla g(x_k)+\frac{1}{\mu (1-\frac{1}{K})^{K-k-1}}\ell\big).
\end{equation*}
In many cases, such projection could be computed in linear time $\O(n)$ \citep{brucker1984n,pardalos1990algorithm}, e.g., for $\K=\{x\in \R^n:1^Tx\leq 1,0\preceq x\preceq 1\}$.\\
Algorithm \ref{alg1} is different from the Frank-Wolfe variant of \citet{bian2017guaranteed} in two important aspects: 1) At step $k\in\{0,\dots,K-1\}$, Algorithm \ref{alg1} is applied to the function $(1-\frac{1}{K})^{K-k-1}g(\cdot)+\langle \ell,\cdot\rangle$ (instead of $f(\cdot)=g(\cdot)+\langle \ell,\cdot \rangle$), 2) The linear maximization step for computing $\{v_k\}_{k=0}^{K-1}$ in the Frank-Wolfe variant of \citet{bian2017guaranteed} is replaced by a strongly concave maximization problem. Modification 1 is inspired by a similar idea in \citet{feldman2021guess} where they provided an algorithm for the setting where the objective function is the multilinear extension of a submodular set function ($\mu=0$) and obtained a $1-\frac{c_f}{e}$ approximation ratio for the problem. They also proved matching negative results showing that no polynomial time algorithm can perform better in terms of the approximation ratio. The same upper bound applies to our framework as well. However, the additional strong DR-submodularity of $f$ allows for a faster convergence to $(1-\frac{c_f}{e})\text{OPT}$. We provide the approximation guarantee of Algorithm \ref{alg1} below.
\begin{algorithm}[t]
    \caption{Strongly DR-submodular Frank-Wolfe (SDRFW)}
	\begin{algorithmic}\label{alg1}
		\STATE \textbf{Input}: Convex set $\K$ with $0\in \K$, $L$-smooth and $\mu$-strongly DR-submodular $f$, $g(\cdot)=f(\cdot)-\langle \ell, \cdot\rangle$ where for all $i\in[n]$, $\ell_i=\min_x \nabla_i f(x)$, $K>0$.
		\STATE Set $x_0=0$.
		\FOR{$k=0$ {\bfseries to} $K-1$}
		\STATE Set $v_k=\argmax_{x\in \K}\langle (1-\frac{1}{K})^{K-k-1}\nabla g(x_k)+\ell,x\rangle -\frac{\mu (1-\frac{1}{K})^{K-k-1}}{2}\|x\|^2$.
		\STATE $x_{k+1}=x_k+\frac{1}{K}v_k$.
		\ENDFOR
		\STATE \textbf{Output}: $x=x_{K}$.
	\end{algorithmic}
\end{algorithm}
\begin{theorem}\label{thm1}
Let $f:\K \rightarrow \mathbb{R}$, $\K\subset \mathbb{R}_+^n$ and $0\in \K$, be a normalized, monotone, $\mu$-strongly DR-submodular and $L$-smooth function. If we set $K=\lceil \frac{L}{\mu}\rceil$, the output of Algorithm \ref{alg1} has the following performance guarantee:
\begin{equation*}
    f(x)\geq (1-\frac{c_f}{e})f(x^*),
\end{equation*}
where $x^*=\argmax_{x\in \K}f(x)$.
\end{theorem}
In comparison, the Frank-Wolfe variant of \citet{bian2017guaranteed} obtains the approximation guarantee $f(x)\geq (1-\frac{1}{e})f(x^*)-\frac{LR^2}{2K}$ where $R=\max_{x,y\in \K}\|y-x\|$ is the diameter of $\K$. Therefore, Algorithm \ref{alg1} has an improved approximation ratio for objective functions $f$ with curvature $c_f<1$ and its guarantee does not deteriorate as the diameter of $\K$ becomes larger. Intuitively, for $K$ chosen large enough (i.e., $K\geq \frac{L}{\mu}$) in the analysis of Algorithm \ref{alg1}, the $-\frac{LR^2}{2K}$ term is cancelled with a positive term resulting from $\mu$-strong DR-submodularity of the objective function.\\
We need to know the smoothness parameter $L$ and the strong DR-submodularity parameter $\mu$ to set $K$ in Algorithm \ref{alg1}. For $\mu$-strongly DR-submodular functions, if we run Algorithm \ref{alg1} with $\hat{\mu}$ instead of $\mu$, the algorithm obtains the same guarantee as long as $\hat{\mu}\leq \mu$. In order to compute such a lower bound, one can investigate the diagonal entries of the Hessian matrix. In Section \ref{sec:5}, we show how to compute $L$ efficiently using convex optimization tools.
\section{Projected Gradient Ascent (PGA) algorithm}\label{sec:4}
In this section, we study the well-known Projected Gradient Ascent (PGA) method for strongly DR-submodular maximization with bounded curvature. The PGA algorithm is provided in Algorithm \ref{alg2} \citep{nesterov2003introductory}. Given an initial point $x_1\in \K$, PGA iteratively applies the update $x_{k+1}=\text{Proj}_{\K}\big(x_k+\frac{1}{L}\nabla f(x_k)\big)$. In other words, at each iteration $k\in[K]$, the current iterate $x_k$ is updated by adding $\frac{1}{L}\nabla f(x_k)$ and the resulting point is then projected back to the constraint set $\K$. The algorithm outputs the final iterate $x=x_{K+1}$. Unlike the SDRFW algorithm, PGA does not require to start from the origin and for any feasible initial point $x_1\in \K$, PGA still converges to a competitive solution. However, as we will soon see in the result of Theorem \ref{thm2}, the rate of convergence depends on the distance between the initial point $x_1$ and the optimal point $x^*$.\\
We first provide a key lemma below that is used in the analysis of Algorithm \ref{alg2}.
\begin{lemma}\label{lemma2}
For any $x,z\in \K$, if $f$ is a non-negative monotone $\mu$-strongly DR-submodular function with curvature $c_f$, we have:
\begin{equation*}
f(z)-(1+c_f)f(x)\leq \langle \nabla f(x),z-x\rangle -\frac{\mu}{2}\|z-x\|^2.    
\end{equation*}
\end{lemma}
\begin{proof}
Let $u=x\vee z$ and $w=x\wedge z$. Using the $\mu$-strong DR-submodularity property of $f$, we can write:
\begin{align*}
    f(u)-f(x)&\leq \langle \nabla f(x),u-x\rangle-\frac{\mu}{2}\|u-x\|^2,\\
    f(w)-f(x)&\leq \langle \nabla f(x),w-x\rangle -\frac{\mu}{2}\|w-x\|^2.
\end{align*}
Taking the sum of the two inequalities and using the fact that $u + w = x + z$, we have:
\begin{equation}\label{eq:22}
f(u)+f(w)-2f(x)\leq \langle \nabla f(x),z-x\rangle -\frac{\mu}{2}\|z-x\|^2.
\end{equation}
Using the mean value theorem, we can write:
\begin{align*}
    f(u)-f(z)&=\int_0^1 \langle u-z, \nabla f\big(z+t(u-z)\big)\rangle dt,\\
    f(w)-f(x)&=\int_0^1 \langle w-x, \nabla f\big(w+t(x-w)\big)\rangle dt.
\end{align*}
Given that $x-w=u-z$ and $\nabla_i f\big(z+t(u-z)\big)\geq (1-c_f)\nabla_i f\big(w+t(x-w)\big)$ holds for all $i\in[n]$, we can bound the first inequality as follows:
\begin{align*}
    f(u)-f(z)&\geq-(1-c_f)\int_0^1 \langle w-x, \nabla f\big(w+t(x-w)\big)\rangle dt\\
    &=-(1-c_f)\big(f(w)-f(x)\big).
\end{align*}
Equivalently, we can write:
\begin{equation}\label{eq:33}
f(u)+f(w)\geq f(z)+(1-c_f)f(x)+c_ff(w).    
\end{equation}
Combining the inequalities \ref{eq:22} and \ref{eq:33}, we conclude:
\begin{equation*}
    f(z)-(1+c_f)f(x)+c_ff(w)\leq \langle \nabla f(x),z-x\rangle -\frac{\mu}{2}\|z-x\|^2.
\end{equation*}
Given that $c_f\in [0,1]$ and $f$ is non-negative, we can drop the term $c_ff(w)$ and derive the result as stated.
\end{proof}
We can now exploit the result of Lemma \ref{lemma2} to obtain the approximation guarantee of the PGA algorithm.
\begin{theorem}\label{thm2}
If $f$ is a monotone, $\mu$-strongly DR-submodular and $L$-smooth function, PGA obtains the following approximation guarantee:
\begin{equation*}
f(x)\geq \frac{1}{1+c_f}f(x^*)-\frac{e^{-\mu K/L}}{1+c_f}\big(f(x^*)-(1+c_f)f(x_1)\big).
\end{equation*}
Moreover, for the DR-submodular setting ($\mu=0$), the utility of the output of the PGA algorithm is bounded as follows:
\begin{equation*}
    f(x)\geq \frac{1}{1+c_f}f(x^*)-\frac{L}{2K(1+c_f)}\|x_1-x^*\|^2.
\end{equation*}
\end{theorem}
\begin{algorithm}[t]
	\caption{Projected Gradient Ascent (PGA)}
	\begin{algorithmic}\label{alg2}
		\STATE \textbf{Input}: Convex set $\K$, $x_1\in \K$, $L$-smooth and $\mu$-strongly DR-submodular $f$, $K>0$.
		\FOR{$k=1$ {\bfseries to} $K$}
		\STATE Set $x_{k+1}=\text{Proj}_{\K}\big(x_k+\frac{1}{L}\nabla f(x_k)\big)$.
		\ENDFOR
		\STATE \textbf{Output}: $x=x_{K+1}$.
	\end{algorithmic}
\end{algorithm}
\citet{hassani2017gradient} analyzed the PGA method in the DR-submodular setting ($\mu=0$) and proved that $f(x)\geq \frac{1}{2}f(x^*)-\frac{LR^2}{2K}$. In comparison, thanks to Lemma \ref{lemma2}, we obtain an improved $\frac{1}{1+c_f}$ approximation ratio with a similar sub-linear convergence rate in the DR-submodular setting. Moreover, if $f$ is $\mu$-strongly DR-submodular, Theorem \ref{thm2} shows that the $\frac{1}{1+c_f}$ approximation ratio could be achieved at a faster linear convergence rate. Furthermore, if $c_f<\frac{1}{e-1}\approx 0.58$, the $\frac{1}{1+c_f}$ approximation ratio obtained in Theorem \ref{thm2} is larger than the $1-\frac{1}{e}$ approximation ratio guaranteed by the Frank-Wolfe variant of \citet{bian2017guaranteed}. However, $\frac{1}{1+c_f}\leq 1-\frac{c_f}{e}$ always holds, i.e., the approximation ratio of the SDRFW algorithm is greater than that of the PGA algorithm.
\subsection{Online setting}\label{sec:4-1}
We can also use Lemma \ref{lemma2} to obtain improved regret bounds in the online setting for the Online Gradient Ascent (OGA) algorithm. To be precise, consider the following online optimization protocol: A convex constraint set $\K$ with diameter $R$ is given. At each iteration $t\in[T]$, the online algorithm first chooses an action $x_t\in \K$. Upon committing to this choice, a monotone (strongly) DR-submodular function $f_t:\K \rightarrow \mathbb{R}$, $\K\subset \mathbb{R}_+^n$, is revealed and the algorithm receives the reward $f_t(x_t)$. The goal is to maximize the total obtained reward or equivalently minimize the $\alpha$-regret $\alpha\text{-}R_T=\alpha \max_{x\in \K}\sum_{t=1}^T f_t(x)-\sum_{t=1}^Tf_t(x_t)$, i.e., the difference between the cumulative reward of the algorithm and the $\alpha$ approximation to that of the best fixed decision in hindsight where $\alpha \in (0,1]$.\\
The Online Gradient Ascent (OGA) algorithm is provided in Algorithm \ref{alg3}. OGA is the online counterpart of the PGA algorithm for the offline setting. Starting from an arbitrary initial point $x_1\in \K$, for all $t\in[T-1]$, OGA uses the update $x_{t+1}=\text{Proj}_{\K}\big(x_t+\eta_t\nabla f_t(x_t)\big)$ to obtain the next iterate $x_{t+1}$, where $\eta_t>0$ is the step size. \citet{pmlr-v84-chen18f} analyzed the OGA algorithm in the DR-submodular setting ($\mu=0$) and provided $\O(\sqrt{T})$ bounds for the $\frac{1}{2}$-regret of the algorithm. Using Lemma \ref{lemma2}, we can obtain improved $\O(\sqrt{T})$ and $\O(\ln T)$ $(\frac{1}{1+c})$-regret bounds in the DR-submodular and strongly DR-submodular settings respectively where $c=\max_{t\in[T]}c_{f_t}$. This result is stated in the theorem below. 
\begin{theorem}\label{thm3}
Assume that the functions $\{f_t\}_{t=1}^T$ are all monotone, $\beta$-Lipschitz and $\mu$-strongly DR-submodular. If for all $t\in [T]$, we set $\eta_t=\frac{1}{\mu t}$, the OGA algorithm has the following $(\frac{1}{1+c})$-regret bound.
\begin{equation*}
    \frac{1}{1+c}\sum_{t=1}^Tf_t(x^*)-\sum_{t=1}^Tf_t(x_t)\leq \frac{\beta^2}{2\mu (1+c)}(1+\ln T),
\end{equation*}
where $x^*=\argmax_{x\in \K}\sum_{t=1}^Tf_t(x)$. Moreover, in the DR-submodular setting ($\mu=0$), if we set $\eta_t=\eta=\frac{R}{\beta \sqrt{T}}~\forall t\in[T]$, the $(\frac{1}{1+c})$-regret of the algorithm is bounded as follows:
\begin{equation*}
    \frac{1}{1+c}\sum_{t=1}^Tf_t(x^*)-\sum_{t=1}^Tf_t(x_t)\leq \frac{R\beta}{1+c} \sqrt{T}.
\end{equation*}
\end{theorem}
For the online monotone DR-submodular maximization problem with bounded curvature, the only prior study was done by \citet{harvey2020improved} where the authors proposed an algorithm for the special setting where the DR-submodular functions $\{f_t\}_{t=1}^T$ are the multilinear extensions of corresponding submodular set functions $\{F_t\}_{t=1}^T$ and they showed that the algorithm obtains an $\O(\sqrt{T})$ $(1-\frac{c}{e}-\epsilon)$-regret bound with $\frac{n^2}{\epsilon}$ projections per iteration. In comparison, while the approximation ratio in Theorem \ref{thm3} is slightly worse, Algorithm \ref{alg3} performs \emph{only a single projection per step} (hence a significantly lower computational complexity) and its logarithmic regret bound is superior in the strongly DR-submodular setting.
\begin{algorithm}[t]
	\caption{Online Gradient Ascent (OGA)}
	\begin{algorithmic}\label{alg3}
		\STATE \textbf{Input}: Convex set $\K$, $x_1\in \K$, $L$-smooth and $\mu$-strongly DR-submodular $\{f_t\}_{t=1}^T$, $\{\eta_t\}_{t=1}^T$.
		\STATE \textbf{Output}: $\{x_t\}_{t=1}^T$.
		\FOR{$t=1$ {\bfseries to} $T$}
		\STATE Play $x_t$ and receive the reward $f_t(x_t)$.
		\STATE Set $x_{t+1}=\text{Proj}_{\K}\big(x_t+\eta_t\nabla f_t(x_t)\big)$.
		\ENDFOR
	\end{algorithmic}
\end{algorithm}
\setcounter{equation}{0}
\section{Computing the smoothness parameter}\label{sec:5}
Both the SDRFW and PGA algorithms require knowledge of the smoothness parameter $L$ of the objective function to be implemented. In this section, we show how computing $L$ of a twice differentiable $\mu$-strongly DR-submodular objective function $f$ ($\mu\geq 0$) could be formulated as a convex optimization problem that could be solved efficiently prior to running the algorithms. Given that our technique applies to the DR-submodular setting (i.e., $\mu=0$) as well, the results of this section are useful for the proposed algorithms in prior works for DR-submodular maximization. In particular, while \citet{hassani2017gradient} chose the step size of (stochastic) PGA as a function of the smoothness parameter $L$ to obtain the theoretical approximation guarantees in their work, they mentioned that estimating $L$ is difficult in general and poses a challenge for implementation. Therefore, they suggested an alternative adaptive step size rule (as a function of the iteration number) with no theoretical performance guarantees in their experiments. This section precisely addresses the aforementioned challenge.\\
As we defined smoothness earlier in Section \ref{sec:2}, we need to find a constant $L>0$ such that for all $x\in \K$, the smallest eigenvalue of $\nabla^2 f(x)$ is lower bounded by $-L$, i.e., $\nabla^2 f(x)\succeq -LI$. This is equivalent to finding $L>0$ such that the largest eigenvalue of $-\nabla^2 f$ is uniformly upper bounded by $L$ everywhere. Given that $f$ is $\mu$-strongly DR-submodular, $-\nabla^2 f$ is an element-wise non-negative matrix. The Perron-Frobenius theorem \citep[Theorem 8.4.4]{horn2012matrix} states that if $-\nabla^2 f$ is irreducible, i.e., the matrix $(I-\nabla^2f)^{n-1}$ is element-wise positive, $-\nabla^2 f$ has a positive real eigenvalue $\lambda_{pf}$ equal to its spectral radius (which is the largest magnitude of its eigenvalues) and an entry-wise positive eigenvector $v_{pf}\succ 0$ corresponding to $\lambda_{pf}$ and therefore, we can set $L=\lambda_{pf}$. In order to check irreducibility of the symmetric matrix $-\nabla^2 f$, we can associate with it an undirected graph $G$ with $n$ vertices labeled $\{1,\dots,n\}$ where there is an edge between vertices $i$ and $j$ if $-\nabla_{ij}^2 f=-\nabla_{ji}^2 f\gneq 0$. $-\nabla^2 f$ is irreducible if and only if its associated graph $G$ is connected. According to a result in the theory of non-negative matrices, the Perron-Frobenius eigenvalue is the solution of the following optimization problem:
\begin{equation}\label{pf}
\begin{array}{ll}
\mbox{minimize}& \lambda\\
\mbox{subject to}& \sum_{j=1}^n -\nabla_{ij}^2f(x)v_j\leq \lambda v_i~\forall i\in[n],
\end{array}
\end{equation}
where the variables are $x\succ 0$, $v\succ 0$ and $\lambda>0$. We show how the above problem could be transformed to a convex optimization problem in many cases.\\
A function $h:\R_{++}^n\to \R$ defined as $h(x)=cx_1^{a_1}\dots x_n^{a_n}$, where $c>0$ and $a_i\in \R~\forall i\in[n]$, is called a monomial. A sum of $m$ monomials, i.e., a function of the form $h(x)=\sum_{s=1}^m c_sx_1^{a_{1s}}\dots x_n^{a_{ns}}$ where $c_s>0~\forall s\in[m]$ and $a_{is}\in \R~\forall i\in[n],s\in[m]$, is called a posynomial. Consider the following optimization problem with variable $x\in \R_{++}^n$:
\begin{equation*}
\begin{array}{ll}
\mbox{minimize}& h_0(x)\\
\mbox{subject to}& h_{1i}(x)\leq 1~\forall i\in[n]\\
& h_{2j}(x)= 1~\forall j\in[p].
\end{array}
\end{equation*}
If $h_0,h_{11},h_{12},\dots,h_{1n}$ are posynomials and $h_{21},\dots,h_{2p}$ are monomials, the above problem is called a Geometric Program (GP). While GPs are not generally convex optimization problems, they can be transformed to convex problems (see \citet[Section 4.5.3]{boyd2004convex} for details).\\
For all $i\in[n]$, we can rewrite the constraints of problem $(\ref{pf})$ in the following equivalent way:
\begin{equation*}
    (\lambda^{-1}v_i^{-1})\sum_{j=1}^n -\nabla_{ij}^2f(x)v_j\leq 1.
\end{equation*}
\noindent If the non-zero entries of $-\nabla^2 f(x)$ are posynomial functions of the variable $x$, the constraints of problem $(\ref{pf})$ are posynomial inequalities and therefore, the optimization problem for computing the smoothness parameter $L=\lambda_{pf}$ can be expressed as a GP.\\
As an example, consider the problem of computing the stability number of an undirected graph $G$ with adjacency matrix $A$ where $f(x)=x^T (-A-I) x+21^Tx$ (introduced in Section \ref{sec:1}) . Without loss of generality, assume that $G$ is connected (otherwise, we can consider each connected component of $G$ separately and use the fact that the stability number of a graph equals the sum of stability numbers of its connected components). Since $G$ is connected, $-\nabla^2f(x)=2A+2I$ is an entry-wise non-negative and irreducible matrix and therefore, computing the smoothness parameter $L$ of $f$ could be formulated as a GP.\\
More generally, consider the class of concave functions with negative dependence. Let $d\geq 2$. If $h_i:\R_+ \to \R$ is $\mu$-strongly concave for all $i\in[n]$ and $\theta_{i_1,\dots,i_r}\leq0$ for all $r\in[d]$ and $(i_1,\dots,i_r)\subseteq [n]$, the following function $f:\R_+^n \to \R$ is $\mu$-strongly DR-submodular:
\begin{equation*}
f(x)=\sum_{i=1}^n h_i (x_i)+\sum_{(i,j):i\neq j}\theta_{ij}x_i x_j+\dots+\sum_{(i_1,\dots,i_d):i_r\neq i_s~\forall r,s\in[d]}\theta_{i_1,\dots,i_d}x_{i_1}\dots x_{i_d}.
\end{equation*}
It is easy to see that all off-diagonal entries of $-\nabla^2 f(x)$ are posynomials. If for all $i\in[n]$, $-h_i''(x_i)$ is a posynomial as well, the smoothness parameter of $f$ could be computed as the solution of a GP.
\section{Numerical examples}\label{sec:6}
\begin{figure}[t]
    \centering
    \includegraphics[scale=0.34]{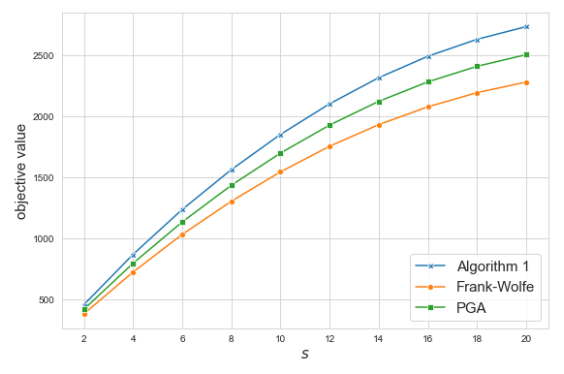}\\
  \caption{Comparison of algorithms for a class of strongly DR-submodular quadratic functions.}\label{fig:1}
\end{figure}
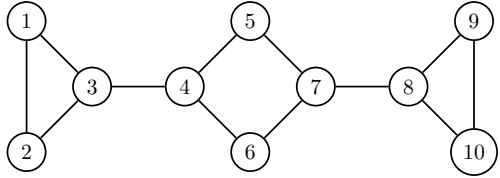
\begin{figure}[b]
\centering
\resizebox{0.4\textwidth}{!}{%
\begin{tikzpicture}[node distance={15mm}, thick, main/.style = {draw, circle}] 
\node[main] (4) {$4$}; 
\node[main] (5) [above right of=4] {$5$}; 
\node[main] (6) [below right of=4] {$6$}; 
\node[main] (7) [above right of=6] {$7$}; 
\node[main] (8) [right of=7] {$8$}; 
\node[main] (9) [above right of=8] {$9$}; 
\node[main] (10) [below right of=8] {$10$}; 
\node[main] (3) [left of=4] {$3$}; 
\node[main] (1) [above left of=3] {$1$}; 
\node[main] (2) [below left of=3] {$2$}; 
\draw (1) -- (2); 
\draw (1) -- (3); 
\draw (2) -- (3); 
\draw (3) -- (4); 
\draw (4) -- (5); 
\draw (4) -- (6); 
\draw (5) -- (7); 
\draw (6) -- (7); 
\draw (7) -- (8); 
\draw (8) -- (9); 
\draw (8) -- (10);
\draw (9) -- (10); 
\end{tikzpicture} 
}%
\caption{Graph $G$ with $n=10$ used in Experiment 2}
\label{fig:2}
\end{figure}
For the first experiment, we set $n=25$ and chose $\K=\{x\in \R^n:1^Tx\leq s,0\preceq x\preceq 1\}$. We considered the class of indefinite quadratic functions and defined $f(x)=(\frac{1}{2}x-\textbf{1})^THx$, where $H\in \R^{n\times n}$ is a matrix whose entries are uniformly distributed in the range $[-10,-5]$. Therefore, $\mu=5$ in this setting. We computed $L$ using the technique described in Section \ref{sec:5} and set $K=\lceil\frac{L}{\mu}\rceil$. We ran Algorithm \ref{alg1}, Frank-Wolfe variant of \citet{bian2017guaranteed} and PGA for $K$ iterations using different values of $s$ in the range $[2,20]$ (using $x_1=0$ for PGA) and plotted the utility of the output of all three algorithms in each setting. The plot is depicted in Figure \ref{fig:1}. This plot shows that Algorithm \ref{alg1} manages to obtain slightly higher utilities, followed by PGA and the Frank-Wolfe variant of \citet{bian2017guaranteed}.

As the second experiment, we studied the problem of computing the stability number $s(G)$ for two graphs. The first graph is provided in Figure \ref{fig:2}. It is easy to see that $s(G)=4$ for this graph (e.g., vertices $\{3,5,6,8\}$ form a maximum stable set of $G$). As it was mentioned earlier in Section \ref{sec:1}, we set $\K=\{x\in \R^{n}:1^Tx=1,x_i\geq 0~\forall i\in[n]\}$ where $n=10$. Also, we defined $f(x)=x^T (-A-I) x+21^Tx$ and using the formula $s(G)^{-1}=2-\max_{x\in \K}f(x)$, we set $\frac{1}{2-f(x)}$ as the estimate of the stability number for $x\in \K$. Since all the diagonal entries of $\nabla^2f(\cdot)$ are equal to $-2$, we have $\mu=2$. We also computed $L$ using the method presented in Section \ref{sec:5}. We ran PGA for this problem and plotted $\frac{1}{2-f(x_k)}$ versus the iteration number $k$ in Figure \ref{fig:3}(a). As the plot shows, PGA converges to the optimal value $4$ after only 10 iterations. As the second example, we considered a graph with $n=1024$ vertices from \url{https://oeis.org/A265032/a265032.html} that contains a collection of graph instances commonly used in coding theory. Using an algorithm with a running time of 200 hours, \citet{niskanen2003cliquer} managed to show that the stability number of this graph is 196. The performance of PGA for this problem is plotted in Figure \ref{fig:3}(b). The algorithm converges to the value 182 as the estimate of the stability number and therefore, the performance of PGA for this problem is significantly better than the approximation guarantee proved in Theorem \ref{thm2}. Note that the domain in the second experiment does not contain the origin and thus, Algorithm \ref{alg1} is not applicable to this problem.
\begin{figure*}
  \centering
  \begin{tabular}[b]{cc}
    \includegraphics[scale=0.34]{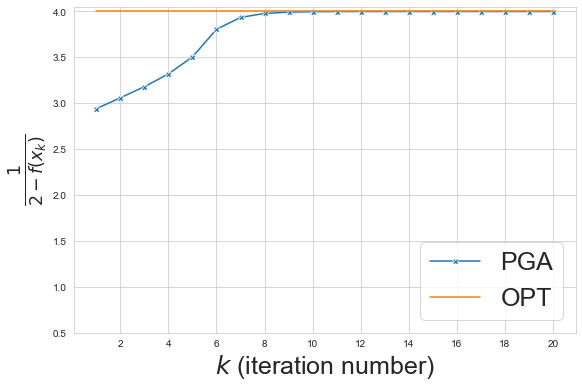} &  \includegraphics[scale=0.34]{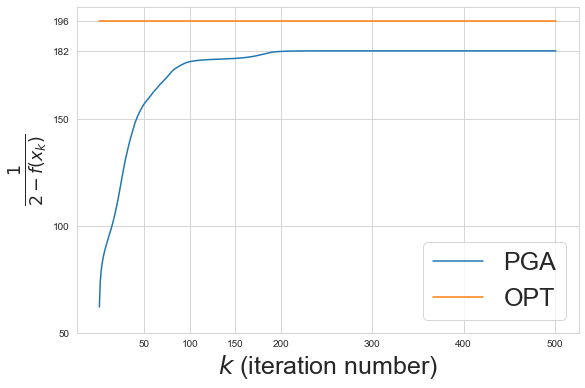}\\
    \small (a) & \small (b)
  \end{tabular}
  \caption{Computing the stability number $s(G)$ for a graph $G$ with (a) $n=10$ vertices and $s(G)=4$, (b) $n=1024$ vertices and $s(G)=196$.}\label{fig:3}
\end{figure*}
\section{Conclusion and future work}\label{sec:7}
In this paper, we considered the class of monotone, $L$-smooth and $\mu$-strongly DR-submodular functions with bounded curvature $c\in[0,1]$ and we proposed a number of first-order gradient methods for this problem along with their approximation guarantees and convergence rates.\\
% In particular, we proposed the SDRFW algorithm in Section \ref{sec:3} that manages to match the provably optimal $1-\frac{c}{e}$ approximation ratio after only $\lceil \frac{L}{\mu}\rceil$ iterations. We also studied the Projected Gradient Ascent (PGA) algorithm for this problem in Section \ref{sec:4} and provided a refined analysis of the algorithm for this setting with a $\frac{1}{1+c}$ approximation ratio and a linear convergence rate. Moreover, we studied the online counterpart of PGA, namely Online Gradient Ascent (OGA), in Section \ref{sec:4-1} and obtained improved $\O(\ln T)$ bounds for the $(\frac{1}{1+c})$-regret of the algorithm for online strongly DR-submodular maximization. Given that both SDRFW and PGA require knowledge of $L$, we showed in Section \ref{sec:5} how computing $L$ for a (strongly) DR-submodular function could be formulated as a convex optimization problem that could be solved efficiently in advance.\\
This work could be extended in a number of interesting directions. Throughout this paper, we assumed that we have access to the exact gradient of the objective function. However, in many applications, it is difficult to compute the gradient exactly, but an unbiased estimate of the gradient can be easily obtained. It is interesting to study stochastic gradient methods for this setting and provide performance guarantees. Secondly, in this paper, we introduced the class of strongly DR-submodular functions with respect to the norm $\|\cdot\|_2$. As we showed in Lemma \ref{lemma1}, strongly DR-submodular functions are strongly concave along non-negative directions with respect to $\|\cdot\|_2$. This definition could be easily extended to other norms as well. For instance, a twice differentiable function $f$ is $\mu$-strongly DR-submodular with respect to the norm $\|\cdot\|_1$ if \emph{all} entries of the Hessian $\nabla^2 f(x)$ are upper bounded by $-\mu$ for all $x$. Similarly, while the smoothness of $f$ over the domain $\K$ was defined with respect to the Euclidean norm, in some cases, $f$ and $\K$ are not well-behaved in the $\|\cdot\|_2$ norm and $L$ scales with the ambient dimension $n$ leading to slower convergence rates for our proposed algorithms in large-scale applications. In such cases, one can study mirror ascent methods that are designed to adapt to smoothness in general norms.
%Bibliography
\bibliography{references}   
\newpage
\section*{Appendix}
\setcounter{equation}{0}
\section*{Examples of strongly DR-submodular functions}
$\bullet$ \textbf{Indefinite quadratic functions.} Let $f(x)=\frac{1}{2}x^T Hx+h^T x+c$ where $H$ is a symmetric matrix. If $H$ is entry-wise non-positive, $f$ is a DR-submodular function and if in addition, $H_{ii}\leq -\mu$ holds for all $i\in[n]$, $f$ is $\mu$-strongly DR-submodular. Such quadratic utility functions have a wide range of applications. In particular, price optimization with continuous prices \citep{NIPS2016_6301} and computing stability number of graphs \citep{motzkin1965maxima} are both non-concave (strongly) DR-submodular quadratic optimization problems.\\
$\bullet$ \textbf{Concave functions with negative dependence.} Let $d\geq 2$. If $h_i:\R_+ \to \R$ is (strongly) concave for all $i\in[n]$ and $\theta_{i_1,\dots,i_r}\leq0$ for all $r\in[d]$ and $(i_1,\dots,i_r)\subseteq [n]$, the following function $f:\R_+^n \to \R$ is (strongly) DR-submodular:
\begin{equation*}
f(x)=\sum_{i=1}^n h_i (x_i)+\sum_{(i,j):i\neq j}\theta_{ij}x_i x_j+\dots+\sum_{(i_1,\dots,i_d):i_r\neq i_s~\forall r,s\in[d]}\theta_{i_1,\dots,i_d}x_{i_1}\dots x_{i_d}.
\end{equation*}
\section*{Missing proofs}
\subsection*{Proof of Lemma 1}
Without loss of generality, assume $v\succeq 0$ holds (analysis for the $v\prec 0$ is similar). For any $z\in \K$, $v\succeq 0$ and $i\in[n]$, we have:
\begin{align*}
    [\nabla^2 f(z)v]_i+\mu v_i&= \nabla^2_{ii}f(z)v_i+\sum_{j\neq i}\nabla^2_{ij}f(z)v_j+\mu v_i\\
    &=\underbrace{(\nabla^2_{ii}f(z)+\mu)}_{\leq 0}\underbrace{v_i}_{\geq 0}+\sum_{j\neq i}\underbrace{\nabla^2_{ij}f(z)}_{\leq 0}\underbrace{v_j}_{\geq 0}\\
    &\leq 0.
\end{align*}
Therefore, $\nabla^2f(z)v+\mu v\preceq 0$ holds. We can use the mean value theorem twice to write:
\begin{align*}
    f(x+v)-f(x)-\langle\nabla f(x),v\rangle&=\int_{0}^1\langle \nabla f(x+tv),v\rangle dt - \langle \nabla f(x),v\rangle\\
    &=\int_{0}^1\langle \nabla f(x+tv)-\nabla f(x),v \rangle dt\\
    &=\int_{0}^1 \langle t\nabla^2f(z_t)v,v\rangle dt,
\end{align*}
where $z_t$ is in the line segment between $x$ and $x+tv$. Combining the above two inequalities, we have:
\begin{align*}
    f(x+v)-f(x)-\langle\nabla f(x),v\rangle&=\int_{0}^1 \langle t\underbrace{(\nabla^2f(z_t)v+\mu v)}_{\preceq 0},\underbrace{v}_{\succeq 0}\rangle dt-\mu\int_{0}^1 t\langle v,v\rangle dt\\
    &\leq\frac{-\mu}{2}\|v\|^2.
\end{align*}
Thus, $f$ is $\mu$-strongly concave along the non-negative direction $v$.
\subsection*{Proof of Theorem 1}
Define $g(x)=f(x)-\ell^Tx$ where for all $i\in[n]$, $\ell_i=\min_x \nabla_i f(x)$. Note that similar to $f$, $g$ is also a normalized, monotone, $\mu$-strongly DR-submodular and $L$-smooth function. For instance, in order to verify the $\mu$-strong concavity of $g$ along non-negative directions (which will be used in the proof), for all $x\in \K$ and $v\succeq 0$ or $v\preceq 0$, we can write:
\begin{align*}
    g(x+v)&=f(x+v)-\langle \ell, x+v\rangle \\
    &\leq f(x)+\langle \nabla f(x),v\rangle -\frac{\mu}{2}\|v\|^2-\langle \ell, x+v\rangle\\
    &= \underbrace{f(x)-\langle \ell, x\rangle}_{=g(x)} +\langle \underbrace{\nabla f(x)-\ell}_{=\nabla g(x)},v\rangle -\frac{\mu}{2}\|v\|^2\\
    &= g(x)+\langle \nabla g(x),v\rangle -\frac{\mu}{2}\|v\|^2,
\end{align*}
where the inequality uses the $\mu$-strong DR-submodularity of $f$.
For all $k\in\{0,\dots,K-1\}$, we can write:
\begin{equation*}
    g(x_{k+1})\geq g(x_k)+\frac{1}{K}\langle \nabla g(x_k), v_k\rangle -\frac{L}{2K^2}\|v_k\|^2.
\end{equation*}
Rearranging the terms, we can write:
\begin{equation}\label{eq:1}
    g(x_{k+1})-g(x_k)\geq \frac{1}{K}\langle \nabla g(x_k), v_k\rangle -\frac{L}{2K^2}\|v_k\|^2.
\end{equation}
For all $k\in\{0,\dots,K\}$, define $\phi (k)=(1-\frac{1}{K})^{K-k}g(x_k)+\langle \ell, x_k\rangle$. For a fixed $k\in\{0,\dots,K-1\}$, we have:
\begin{align*}
    K\big(\phi (k+1)-\phi (k)\big)&=K(1-\frac{1}{K})^{K-k-1}g(x_{k+1})-K(1-\frac{1}{K})^{K-k}g(x_k)+\langle \ell, v_k\rangle\\
    &=\frac{(1-\frac{1}{K})^{K-k-1}\big(g(x_{k+1})-g(x_k)\big)+1/K(1-\frac{1}{K})^{K-k-1}g(x_{k})}{1/K}+\langle \ell, v_k\rangle\\
    &=\frac{(1-\frac{1}{K})^{K-k-1}\big(g(x_{k+1})-g(x_k)\big)}{1/K}+(1-\frac{1}{K})^{K-k-1}g(x_{k})+\langle \ell, v_k\rangle\\
    &\overset{\text{(a)}}\geq\frac{(1-\frac{1}{K})^{K-k-1}\big(1/K\langle \nabla g(x_k), v_k\rangle -L\|v_k\|^2/2K^2\big)}{1/K}+(1-\frac{1}{K})^{K-k-1}g(x_{k})+\langle \ell, v_k\rangle\\
    &=(1-\frac{1}{K})^{K-k-1}\big(\langle \nabla g(x_k), v_k\rangle -\frac{L}{2K}\|v_k\|^2\big)+(1-\frac{1}{K})^{K-k-1}g(x_{k})+\langle \ell, v_k\rangle\\
    &=\langle(1-\frac{1}{K})^{K-k-1}\nabla g(x_k)+\ell,v_k\rangle -(1-\frac{1}{K})^{K-k-1}\frac{L}{2K}\|v_k\|^2+(1-\frac{1}{K})^{K-k-1}g(x_{k})\\
    &\overset{\text{(b)}}\geq(1-\frac{1}{K})^{K-k-1}\langle \nabla g(x_k),x^*\rangle +\langle \ell, x^*\rangle+\frac{\mu (1-\frac{1}{K})^{K-k-1}}{2}\|v_k\|^2\\
    &-\frac{\mu(1-\frac{1}{K})^{K-k-1} }{2}\|x^*\|^2-(1-\frac{1}{K})^{K-k-1}\frac{L}{2K}\|v_k\|^2+(1-\frac{1}{K})^{K-k-1}g(x_{k}),
\end{align*}
where (a) follows from inequality \ref{eq:1} and (b) is due to the update rule of the SDRFW algorithm for $v_k$. We can use the monotonocity and strong DR-submodularity of $g$ respectively to write:
\begin{align*}
    g(x^*)-g(x_k)&\leq g(x^*+x_k)-g(x_k)\\
    &\leq \langle \nabla g(x_k),x^*\rangle - \frac{\mu}{2}\|x^*\|^2.
\end{align*}
Putting the above two inequalities together, we have:
\begin{align*}
    K\big(\phi (k+1)-\phi (k)\big)&\geq (1-\frac{1}{K})^{K-k-1}\big(g(x^*)-g(x_k)+\frac{\mu}{2}\|x^*\|^2\big) +\langle \ell,x^*\rangle+\frac{\mu (1-\frac{1}{K})^{K-k-1}}{2}\|v_k\|^2\\
    &-\frac{\mu(1-\frac{1}{K})^{K-k-1} }{2}\|x^*\|^2-(1-\frac{1}{K})^{K-k-1}\frac{L}{2K}\|v_k\|^2+(1-\frac{1}{K})^{K-k-1}g(x_{k})\\
    &=(1-\frac{1}{K})^{K-k-1}g(x^*)+\langle \ell,x^*\rangle+(1-\frac{1}{K})^{K-k-1}\big(\frac{\mu}{2}-\frac{L}{2K}\big)\|v_k\|^2
\end{align*}
Therefore, if we set $K=\lceil\frac{L}{\mu}\rceil$ and divide both sides by $K$, we obtain:
\begin{equation*}
    \phi (k+1)-\phi (k)\geq \frac{1}{K}(1-\frac{1}{K})^{K-k-1}g(x^*)+\frac{1}{K}\langle \ell,x^*\rangle.
\end{equation*}
Applying the inequality for all $k\in\{0,\dots,K-1\}$ and taking the sum, we have:
\begin{align*}
    f(x_K)=g(x_K)+\langle \ell, x_K\rangle=\phi(K)-\phi(0)&\geq \big(\frac{1}{K}\sum_{k=0}^{K-1}(1-\frac{1}{K})^{K-k-1}\big)g(x^*)+\langle \ell,x^*\rangle\\
    &=\big(\frac{1}{K}\frac{1-(1-\frac{1}{K})^K}{1/K}\big)g(x^*)+\langle \ell,x^*\rangle\\
    &\geq (1-\frac{1}{e})g(x^*)+\langle \ell,x^*\rangle,
\end{align*}
where the last inequality uses $(1-\frac{1}{K})^K\leq \frac{1}{e}$. Therefore, $f(x)\geq (1-\frac{1}{e})g(x^*)+\langle \ell,x^*\rangle$ holds.\\
Using the mean value theorem, we have $f(x^*)=\sum_{i=1}^n \nabla_i f(u) x^*_i$ where $u$ is in the line segment between 0 and $x^*$. Therefore, we can use the definition of $\ell$ and the curvature to write:
\begin{equation*}
    \ell^Tx^*=\sum_{i=1}^n \ell_i x_i^*\geq (1-c_f)\sum_{i=1}^n \nabla_i f(u)x_i^*=(1-c_f)f(x^*).
\end{equation*}
Putting the above inequalities together, we have:
\begin{align*}
    f(x)&\geq (1-\frac{1}{e})g(x^*)+\ell^Tx^*\\
    &= (1-\frac{1}{e})f(x^*)-(1-\frac{1}{e})\ell^Tx^*+\ell^Tx^*\\
    &= (1-\frac{1}{e})f(x^*)+\frac{1}{e}\ell^T x^*\\
    &\geq (1-\frac{1}{e})f(x^*)+\frac{1-c_f}{e}f(x^*)\\
    &= (1-\frac{c_f}{e})f(x^*).
\end{align*}
\subsection*{Proof of Theorem 2}
We first provide a lemma that will be used in the proof.
\begin{lemma}\label{lemma3}
If we apply PGA with the update rule $x_{k+1}=\text{Proj}_{\K}\big(x_k+\frac{1}{L}\nabla f(x_k)\big)~\forall k\in[K]$ to an $L$-smooth function $f$, the following holds:
\begin{equation*}
    f(x_{k+1})\geq f(x_k)+\frac{L}{2}\|x_{k+1}-x_k\|^2\geq f(x_k).
\end{equation*}
\end{lemma}
\begin{proof}
We can use the $L$-smoothness of $f$ and write:
\begin{align*}
    f(x_{k+1})&\geq f(x_k)+\langle \nabla f(x_k),x_{k+1}-x_k\rangle -\frac{L}{2}\|x_{k+1}-x_k\|^2\\
    &\geq f(x_k)+L\|x_{k+1}-x_k\|^2-\frac{L}{2}\|x_{k+1}-x_k\|^2\\
    &= f(x_k)+\frac{L}{2}\|x_{k+1}-x_k\|^2\\
    &\geq f(x_k),
\end{align*}
where the second inequality follows from the optimality condition for $x_{k+1}$, i.e., $x_{k+1}=\argmin_{z\in \K}\|z-x_k-\frac{1}{L}\nabla f(x_k)\|^2$.
\end{proof}
First, consider the DR-submodular setting ($\mu=0$). Note that for all $k\in[K]$, we can rewrite the update rule of PGA in the following equivalent way:
\begin{equation*}
    x_{k+1}=\argmax_{x\in \K}\big(f(x_k)+\langle \nabla f(x_k),x-x_k\rangle -\frac{L}{2}\|x-x_k\|^2\big)=\text{Proj}_{\K}\big(x_k+\frac{1}{L}\nabla f(x_k)\big).
\end{equation*}
If we denote $h(x):=f(x_k)+\langle \nabla f(x_k),x-x_k\rangle -\frac{L}{2}\|x-x_k\|^2$, $h$ is $L$-strongly concave. Therefore, we can write:
\begin{align*}
    h(x^*)&\leq h(x_{k+1})+\langle \nabla h(x_{k+1}),x^*-x_{k+1}\rangle - \frac{L}{2}\|x_{k+1}-x^*\|^2\\
    &\overset{\text{(a)}}\leq h(x_{k+1}) - \frac{L}{2}\|x_{k+1}-x^*\|^2,
\end{align*}
where (a) follows from the optimality condition for $x_{k+1}$, i.e., $x_{k+1}=\argmax_{x\in \K}h(x)$. We can write:
\begin{align*}
    f(x_{k+1})&\overset{\text{(b)}}\geq h(x_{k+1})\\
    &\geq h(x^*)+\frac{L}{2}\|x_{k+1}-x^*\|^2\\
    &=f(x_k)+\langle \nabla f(x_k),x^*-x_k\rangle - \frac{L}{2}\|x_k-x^*\|^2+\frac{L}{2}\|x_{k+1}-x^*\|^2\\
    &\overset{\text{(c)}}\geq f(x^*)-c_ff(x_k)+\frac{L}{2}\big(\|x_{k+1}-x^*\|^2-\|x_{k}-x^*\|^2\big),
\end{align*}
where (b) follows from the $L$-smoothness of $f$ and (c) uses the result of Lemma 2 with $z=x^*$, $x=x_k$ and $\mu=0$. Rearranging the terms and taking the sum over $k\in[K]$, we obtain:
\begin{align*}
    K(1+c_f)f(x_{K+1})&\overset{\text{(d)}}\geq \sum_{k=1}^K \big(f(x_{k+1})+c_ff(x_k)\big)\\
    &\geq Kf(x^*)+\frac{L}{2}\|x_{K+1}-x^*\|^2-\frac{L}{2}\|x_1-x^*\|^2\\
    &\geq Kf(x^*)-\frac{L}{2}\|x_1-x^*\|^2,
\end{align*}
where (d) is due to Lemma \ref{lemma3}. Dividing both sides by $K(1+c_f)$, we derive the result as stated.\\
Now, we move on to the general strongly DR-submodular setting with $\mu>0$. Note that for all $k\in[K]$, we can equivalently write the update rule of PGA as follows:
\begin{equation*}
    x_{k+1}=\argmax_{x\in \K}\big(\langle \nabla f(x_k),x-x_k\rangle -\frac{L}{2}\|x-x_k\|^2\big)=\text{Proj}_{\K}\big(x_k+\frac{1}{L}\nabla f(x_k)\big).
\end{equation*}
Using the $L$-smoothness property of $f$, we can write:
\begin{align}
    f(x_{k+1})-f(x_k)&\geq \max_{x\in \K}\big(\langle \nabla f(x_k),x-x_k\rangle -\frac{L}{2}\|x-x_k\|^2\big)\nonumber\\
    &= \frac{1}{2L}\|\nabla f(x_k)\|^2-\frac{L}{2}\min_{x\in \K}\|x-x_k-\frac{1}{L}\nabla f(x_k)\|^2\nonumber\\
    &= \frac{1}{2L}\|\nabla f(x_k)\|^2-\frac{L}{2}\|x_{k+1}-x_k-\frac{1}{L}\nabla f(x_k)\|^2\label{eq:3}.
\end{align}
Setting $z=x^*$ and $x=x_k$ in Lemma 2, we have:
\begin{align}
    \frac{\mu}{L}\big(f(x^*)-(1+c_f)f(x_k)\big)&\leq \frac{\mu}{L}\big(\langle \nabla f(x_k),x^*-x_k\rangle-\frac{\mu}{2}\|x^*-x_k\|^2\big)\nonumber\\
    &=-\frac{\mu^2}{2L}\|x^*-x_k-\frac{1}{\mu}\nabla f(x_k)\|^2+\frac{1}{2L}\|\nabla f(x_k)\|^2.\label{eq:4}
\end{align}
Since $x_{k+1}=\argmin_{z\in \K}\|z-x_k-\frac{1}{L}\nabla f(x_k)\|^2$, $0\leq \frac{\mu}{L}\leq 1$ and $\K$ is convex, we can write:
\begin{align*}
    \|x_{k+1}-x_k-\frac{1}{L}\nabla f(x_k)\|^2&\leq \|\frac{\mu}{L}x^*+(1-\frac{\mu}{L})x_k-x_k-\frac{1}{L}\nabla f(x_k)\|^2\\
    &= \|\frac{\mu}{L}(x^*-x_k)-\frac{1}{L}\nabla f(x_k)\|^2\\
    &= \frac{\mu^2}{L^2}\|x^*-x_k-\frac{1}{\mu}\nabla f(x_k)\|^2.
\end{align*}
Multiplying both sides of the above inequality by $-\frac{L}{2}$, we obtain:
\begin{equation}\label{eq:5}
  -\frac{\mu^2}{2L}\|x^*-x_k-\frac{1}{\mu}\nabla f(x_k)\|^2 \leq -\frac{L}{2}\|x_{k+1}-x_k-\frac{1}{L}\nabla f(x_k)\|^2.
\end{equation}
Combining the inequalities \ref{eq:3}, \ref{eq:4} and \ref{eq:5}, we have:
\begin{equation*}
    \frac{\mu}{L}\big(f(x^*)-(1+c_f)f(x_k)\big)\leq f(x_{k+1})-f(x_k).
\end{equation*}
Therefore, we can write:
\begin{equation*}
    f(x^*)-(1+c_f)f(x_{k+1})\leq (1-\frac{\mu}{L})\big(f(x^*)-(1+c_f)f(x_k)\big)+c_f\big(f(x_k)-f(x_{k+1})\big)\leq (1-\frac{\mu}{L})\big(f(x^*)-(1+c_f)f(x_k)\big),
\end{equation*}
where the last inequality uses the result of Lemma \ref{lemma3}. Applying the above inequality recursively for all $k\in[K]$, we obtain:
\begin{align*}
    f(x^*)-(1+c_f)f(x_{K+1})&\leq (1-\frac{\mu}{L})^K\big(f(x^*)-(1+c_f)f(x_1)\big)\\
    &\leq e^{-\mu K/L}\big(f(x^*)-(1+c_f)f(x_1)\big).
\end{align*}
Rearranging the terms and dividing both sides by $1+c_f$, we obtain the result as stated.
\subsection*{Proof of Theorem 3}
We can write:
\begin{align*}
    \|x_{t+1}-x^*\|^2&\leq \|x_t+\eta_t\nabla f_t(x_t)-x^*\|^2\\
    &=\|x_t-x^*\|^2+\eta_t^2\|\nabla f_t(x_t)\|^2-2\eta_t\langle \nabla f_t(x_t),x^*-x_t\rangle.
\end{align*}
Rearranging the terms and dividing both sides by $2\eta_t$, we can equivalently write:
\begin{align}
\langle \nabla f_t(x_t),x^*-x_t\rangle &\leq \frac{\|x_t-x^*\|^2-\|x_{t+1}-x^*\|^2+\eta_t^2\|\nabla f_t(x_t)\|^2}{2\eta_t}\nonumber\\
&\leq \frac{\|x_t-x^*\|^2-\|x_{t+1}-x^*\|^2}{2\eta_t}+\frac{\eta_t\beta^2}{2}\label{eq:6}.
\end{align}
Using the result of Lemma 2 with $x=x_t$, $z=x^*$, $c_f=\max_{t\in[T]}c_{f_t}$, and $\mu=0$ and combining it with the above inequality, we have:
\begin{equation*}
    f_t(x^*)-(1+c_f)f_t(x_t)\leq \frac{\|x_t-x^*\|^2-\|x_{t+1}-x^*\|^2}{2\eta_t}+\frac{\eta_t\beta^2}{2}.
\end{equation*}
Setting $\eta_t=\eta=\frac{R}{\beta \sqrt{T}}~\forall t\in[T]$ and taking the sum over $t\in[T]$, we obtain:
\begin{align*}
    \sum_{t=1}^T\big(f_t(x^*)-(1+c_f)f_t(x_t)\big)&\leq \sum_{t=1}^T \frac{\|x_t-x^*\|^2-\|x_{t+1}-x^*\|^2}{2\eta}+\frac{\eta \beta^2 T}{2}\\
    &=\frac{\|x_1-x^*\|^2-\|x_{T+1}-x^*\|^2}{2\eta}+\frac{\eta \beta^2 T}{2}\\
    &\leq\frac{R^2}{2\eta}+\frac{\eta \beta^2 T}{2}.\\
\end{align*}
Plugging in the value of $\eta$ and dividing both sides by $1+c_f$, we conclude:
\begin{equation*}
    \sum_{t=1}^T\big(\frac{1}{1+c_f}f_t(x^*)-f_t(x_t)\big)\leq \frac{R\beta \sqrt{T}}{2(1+c_f)}+\frac{R\beta \sqrt{T}}{2(1+c_f)}=\frac{R\beta \sqrt{T}}{1+c_f}.
\end{equation*}
For the setting where all $\{f_t\}_{t=1}^T$ are $\mu$-strongly DR-submodular, we can combine the result of Lemma 2 and inequality \ref{eq:6} to obtain the following:
\begin{equation*}
    f_t(x^*)-(1+c_f)f_t(x_t)\leq \frac{\|x_t-x^*\|^2-\|x_{t+1}-x^*\|^2}{2\eta_t}-\frac{\mu}{2}\|x_t-x^*\|^2+\frac{\eta_t\beta^2}{2}.
\end{equation*}
If we set $\eta_t=\frac{1}{\mu t}~\forall t\in[T]$, we have $\frac{1}{2\eta_1}-\frac{\mu}{2}=0$ and $\frac{1}{2\eta_t}-\frac{\mu}{2}=\frac{\mu(t-1)}{2}=\frac{1}{2\eta_{t-1}}~\forall t>1$. Therefore, we can rewrite the above inequality in the following equivalent way:
\begin{equation*}
    f_t(x^*)-(1+c_f)f_t(x_t)\leq \frac{\|x_t-x^*\|^2}{2\eta_{t-1}}-\frac{\|x_{t+1}-x^*\|^2}{2\eta_t}+\frac{\eta_t\beta^2}{2}. 
\end{equation*}
Taking the sum over $t\in[T]$, we have:
\begin{align*}
    \sum_{t=1}^T\big(f_t(x^*)-(1+c_f)f_t(x_t)\big)&\leq -\frac{1}{2\eta_1}\|x_2-x^*\|^2+\sum_{t=2}^T\big(\frac{\|x_t-x^*\|^2}{2\eta_{t-1}}-\frac{\|x_{t+1}-x^*\|^2}{2\eta_t}\big)+\frac{\beta^2}{2}\sum_{t=1}^T \eta_t\\
    &=-\frac{1}{2\eta_1}\|x_2-x^*\|^2+\frac{1}{2\eta_1}\|x_2-x^*\|^2-\frac{1}{2\eta_T}\|x_{T+1}-x^*\|^2+\frac{\beta^2}{2}\sum_{t=1}^T \eta_t\\
    &\leq \frac{\beta^2}{2}\sum_{t=1}^T \eta_t\\
    &\leq \frac{\beta^2}{2\mu}(1+\ln T).
\end{align*}
Dividing both sides by $1+c_f$, we obtain the regret bound as follows:
\begin{equation*}
    \sum_{t=1}^T\big(\frac{1}{1+c_f}f_t(x^*)-f_t(x_t)\big)\leq \frac{\beta^2}{2\mu (1+c_f)}(1+\ln T).
\end{equation*}
\end{document}